\documentclass[11pt]{article} 
\pdfoutput=1
\usepackage{comment}
\usepackage[margin=1in]{geometry}
\usepackage{amsthm}
\usepackage{amsmath}
\usepackage{amssymb}
\usepackage{graphicx}
\usepackage{url}
\usepackage{setspace}
\usepackage{algorithm,algorithmic}
\usepackage{hyperref}
\usepackage{framed}
\usepackage{xcolor}
\usepackage{soul}
\usepackage{natbib}
\usepackage[autostyle]{csquotes}  
\usepackage{caption}
\usepackage{subcaption}
\usepackage[compact]{titlesec}

\usepackage{hyperref}
\hypersetup{
     colorlinks   = true,
      linkcolor = red,
     urlcolor    = magenta,
	 citecolor = blue
}

\usepackage{enumitem}
\usepackage{varwidth}
\usepackage{graphicx}
\usepackage{subcaption}
\usepackage{wrapfig}
\usepackage{soul}
\usepackage{times}

\newtheorem{theorem}{Theorem}[section]
\newtheorem{lemma}[theorem]{Lemma}
\newtheorem{claim}[theorem]{Claim}

\newtheorem{prop}{Proposition}[section]

\newcommand{\D}{\mathcal{D}}
\renewcommand{\P}{\mathcal{P}}
\newcommand{\N}{\mathcal{N}}
\newcommand{\X}{\mathcal{X}}
\newcommand{\E}{\mathbb{E}}
\newcommand{\R}{\mathbb{R}}

\newcommand{\spn}{\mathrm{span}}
\newcommand{\nul}{\mathrm{null}}

\newcommand{\rank}{\mathrm{rank}}
\newcommand{\poly}{\mathrm{poly}}
\newcommand{\polylog}{\mathrm{polylog}}

\newcommand{\dist}{\mathrm{dist}}
\newcommand{\convh}{\mathrm{CH}}

\renewcommand{\vec}{\mathbf}
\newcommand{\hvec}[1]{\hat{\vec #1}}

\newcommand{\near}{\epsilon'/(8 k \sizeA)}
\newcommand{\sizeA}{\alpha}
\newcommand{\pg}{\gamma}

\renewcommand{\parallel}{{\mkern3mu\vphantom{\bot}\vrule depth 0pt\mkern2mu\vrule depth 0pt\mkern3mu}}

\usepackage{enumitem}
\setenumerate{noitemsep,topsep=0pt,parsep=0pt,partopsep=0pt}
\setitemize{noitemsep,topsep=0pt,parsep=0pt,partopsep=0pt}

\title{Generalized Topic Modeling}

\author{Avrim Blum\thanks{Supported in part by National Science Foundation grants CCF-1525971 and CCF-1535967.} \\ \small{avrim@cs.cmu.edu} \and
Nika Haghtalab\thanks{Supported in part by National Science Foundation grant CCF-1525971 and by a Microsoft Research Graduate Fellowship and an IBM Ph.D Fellowship.}\\ \small{nhaghtal@cs.cmu.edu} 
}

\allowdisplaybreaks

\begin{document}

\maketitle

\begin{abstract}
Recently there has been significant activity in developing algorithms with provable guarantees for topic modeling.
In standard topic models, a topic (such as sports, business, or politics) is viewed as a probability distribution $\vec a_i$ over words, and a document is generated by first selecting a mixture $\vec w$ over topics, and then generating words i.i.d. from the associated mixture $A\vec w$.  Given a large collection of such documents, the goal is to recover the topic vectors and then to correctly classify new documents according to their topic mixture.  

In this work we consider a broad generalization of this framework in which words are no longer assumed to be drawn i.i.d. and instead a topic is a complex distribution over sequences of paragraphs. Since one could not hope to even represent such a distribution in general (even if paragraphs are given using some natural feature representation), we aim instead to directly learn a document classifier. That is, we aim to learn a predictor that given a new document, accurately predicts its topic mixture, without learning the distributions explicitly.  We present several natural conditions under which one can do this  efficiently and discuss issues such as noise tolerance and sample complexity in this model. More generally, our model can be viewed as a generalization of the multi-view or co-training setting in machine learning.
\end{abstract}

\setcounter{page}{0}
\thispagestyle{empty}
\newpage
\section{Introduction}

Topic modeling is an area with significant recent work in the intersection of algorithms and machine learning \citep{arora2012computing,arora2012learning,arora2013practical,anandkumar2012spectral,anandkumar2014tensor, bansal2014provable}.  In topic modeling, a topic (such as sports, business, or politics) is modeled as a probability distribution over words, expressed as a vector $\vec a_i$.  A document is generated by first selecting a mixture $\vec w$ over topics, such as 80\% sports and 20\% business, and then choosing words i.i.d. from the associated mixture distribution, which in this case would be $0.8 \vec a_{sports} + 0.2 \vec a_{business}$.  Given a large collection of such documents (and some assumptions about the distributions $\vec a_i$ as well as the distribution over mixture vectors $\vec w$) the goal is to recover the topic vectors $\vec a_i$ and then to use the $\vec a_i$ to correctly classify new documents according to their topic mixtures.

Algorithms for this problem have been developed with strong provable guarantees even when documents consist of only two or three words each \cite{arora2012learning,anandkumar2012spectral,papadimitriou1998latent}.  In addition, algorithms based on this problem formulation perform well empirically on standard datasets \cite{blei2003latent,hofmann1999probabilistic}.

As a theoretical model for document generation, however, an obvious problem with the standard topic modeling framework is that documents are not actually created by independently drawing words from some distribution.  Better would be a model in which {\em sentences} are drawn i.i.d. from a distribution over sentences (this would at least produce grammatical objects and allow for meaningful correlation among related words within a topic, like {\sf shooting} a {\sf free throw} or {\sf kicking} a {\sf field goal}). Even better would be  {\em paragraphs} drawn i.i.d. from a distribution over paragraphs (this would at least produce coherent paragraphs).  Or, even better, how about a model in which paragraphs are drawn non-independently, so that the second paragraph in a document can depend on what the first paragraph was saying, though presumably with some amount of additional entropy as well?  This is the type of model we study here.

Note that an immediate problem with considering such a model is that now the task of learning an explicit distribution (over sentences or paragraphs) is hopeless.  While a distribution over words can be reasonably viewed as a probability vector, one could not hope to learn or even represent an explicit distribution over sentences or paragraphs.  Indeed, except in cases of plagiarism, one would not expect to see the same paragraph twice in the entire corpus. Moreover, this is likely to be true even if we assume paragraphs have some natural feature-vector representation.
Instead, we bypass this issue by aiming to directly learn a predictor for documents---that is, a function that given a document, predicts its mixture over topics---without explicitly learning topic distributions.  Another way to think of this is that our goal is not to learn a model that could be used to {\em write} a new document, but instead just a model that could be used to {\em classify} a document written by others.  This is much as in standard supervised learning where algorithms such as SVMs learn a decision boundary (such as a linear separator) for making predictions on the labels of examples without explicitly learning the distributions $D_+$ and $D_-$ over positive and negative examples respectively.  However, our setting is {\em un}supervised (we are not given labeled data containing the correct classifications of the documents in the training set) and furthermore, rather than each data item belonging to one of the $k$ classes (topics), each data item belongs to a {\em mixture} of the $k$ topics.  Our goal is given a new data item to output what that mixture is.

We begin by describing our high level theoretical formulation.  This formulation can be viewed as a generalization both of standard topic modeling and of a setting known as {\em multi-view learning} or {\em co-training} \cite{blum1998combining,dasgupta2002pac,SSL10,balcan2004co,sun13}.  We then describe several natural assumptions under which we can indeed efficiently solve the problem, learning accurate topic mixture predictors.

\section{Preliminaries} \label{sec:model}

We assume that paragraphs are described by $n$ real-valued features and so can be viewed as points $\vec x$ in an instance space $\X \subseteq \R^n$.
We assume that each document consists of at least two paragraphs and denote it  by $(\vec x^1, \vec x^2)$. 
Furthermore, we consider $k$  topics and partial membership functions $f_1, \dots, f_k: \X \rightarrow[0,1]$, such that $f_i(\vec x)$ determines the degree to which paragraph  $\vec x$ belongs to topic $i$, and, $\sum_{i=1}^k f_i(\vec x) = 1$.
For any vector of probabilities $\vec w \in \R^k$ --- which we sometimes refer to as mixture weights --- we define $\X^{\vec w} = \{ \vec x\in \R^n \mid \forall i,~ f_i(\vec x) = w_i\}$ to be the set of all paragraphs with partial membership values $\vec w$.
We assume that both paragraphs of a document have the same partial membership values, that is $(\vec x^1, \vec x^2) \in \bigcup_{\vec w} \X^{\vec w}\times \X^{\vec w}$, although we also allow some noise later on.  To better relate to the literature on multi-view learning, we will also refer to topics as ``classes'' and refer to paragraphs as ``views'' of the document.

Much like the standard topic models, we consider an unlabeled sample set that is generated by a two-step process. First, we consider a distribution $\P$ over vectors of mixture weights and draw $\vec w$ according to $\P$.
Then we consider distribution $\D^{\vec w}$ over the set $\X^{\vec w} \times \X^{\vec w}$ and draw a document $(\vec x^1, \vec x^2)$ according to $\D^{\vec w}$.
We consider two settings. In the first setting, which is addressed in Section~\ref{sec:no-noise}, the learner receives the instance $(\vec x^1, \vec x^2)$.
In the second setting, the learner receives samples $(\hvec x^1, \hvec x^2)$ that have been perturbed by some noise. We discuss two noise models in Sections~\ref{sec:noise} and \ref{sec:agnostic}.
In both cases, the goal of the learner is to recover the partial membership functions $f_i$.

More specifically, in this work we consider partial membership functions  of the form 
$f_i(\vec x) = f(\vec v_i \cdot \vec x)$, where  $\vec v_1, \dots, \vec v_k \in \R^n$ are linearly independent  and $f:\R \rightarrow [0,1]$  is a monotonic function.
For the majority of this work, we consider $f$ to be the identity function, so that $f_i(\vec x) = \vec v_i \cdot \vec x$.
Define $\vec a_i \in \spn\{\vec v_1, \dots, \vec v_k\}$ such that  $\vec v_i \cdot  \vec a_i=1$ and $\vec v_j \cdot  \vec  a_i = 0$ for all $j\neq i$. That is, $\vec a_i$ can be viewed as the projection of a paragraph that is purely of topic $i$ onto the span of $\vec v_1, \dots, \vec v_k$.  Define $\Delta = \convh(\{\vec a_1, \dots, \vec a_k\})$ to be the convex hull of $\vec a_1, \dots, \vec a_k$.

Throughout this work, we use $\|\cdot \|_2$ to denote the spectral norm of a matrix or the $L_2$ norm of a vector. When  it is clear from the context, we simply use $\|\cdot \|$ to denote these quantities. 
We denote by $B_r(\vec x)$  the ball of radius $r$ around  $\vec x$.  
For any matrix $M$, we use $M^+$ to denote the pseudoinverse of $M$.

\subsection*{Generalization of Standard Topic Modeling}
Let us briefly discuss how the above model is a generalization of the standard topic modeling framework. 
In the standard framework, a topic is modeled as a probability distribution over $n$ words, expressed as a vector $\vec a_i\in [0,1]^n$, where $a_{ij}$ is the probability of word $j$ in topic $i$. 
A document is generated by first selecting a mixture $\vec w\in[0,1]^k$ over $k$ topics, and then choosing words i.i.d. from the associated mixture distribution $\sum_{i=1}^k w_i \vec a_i$.  
The document vector $\hvec x$ is then the vector of word counts, normalized by dividing by the number of words in the document so that the $L_1$ norm of $\hvec x$ is 1.

As a thought experiment, consider infinitely long documents. In the standard framework, all infinitely long documents of a mixture weight $\vec w$ have the same representation $\vec  x = \sum_{i=1}^k w_i \vec a_i$. 
This representation implies $\vec x \cdot \vec v_i = w_i$ for all $i\in [k]$, where $V = (\vec v_1, \dots, \vec v_k)$ is the pseudo-inverse of matrix $A =  (\vec a_1, \dots, \vec a_k)$.
Thus, by partitioning the document into two halves (views) $\vec x^1$ and $\vec x^2$, 
our \emph{noise-free model} with $f_i(\vec x)  =\vec v_i \cdot \vec x$ generalizes the standard topic model for long documents.
However, our model is substantially more general: 
features within a view can be arbitrarily  correlated, the views themselves can be  correlated with each other, and even in the zero-noise case, documents of the same mixture can look very different so long as they have the same projection to the span of the $\vec a_1, \dots, \vec a_k$.

For a shorter document $\hvec x$, each feature $\hat x_i$ is drawn according to a distribution with mean $x_i$, where  $\vec x  = \sum_{i=1}^k w_i \vec a_i$.  Therefore, $\hvec x$ can be thought of as a noisy measurement of $\vec x$. The fewer the words in a document, the larger is the noise in  $\hvec x$. Existing work in topic modeling, such as~\cite{arora2012learning,anandkumar2014tensor}, provide elegant procedures for handling large noise that is caused by drawing only $2$ or $3$ words according to the distribution induced by $\vec x$. As we show in Section~\ref{sec:noise}, our method can also tolerate large amounts of noise under some conditions. 
While our method cannot deal with documents that are only $2$- or $3$-words long, the benefit is a model that is much more general in many other respects.

\section{An Easier Case with Simplifying Assumptions} \label{sec:no-noise}
We make two main simplifying assumptions in this section, both of which will be relaxed in Section~\ref{sec:noise}: 1) The documents are not noisy, i.e., $\vec x^1\cdot \vec v_i  = \vec x^2\cdot\vec v_i$; 2) There is 
non-negligible probability density on instances that belong purely to one class. 
In this section we demonstrate ideas and techniques, which we will develop further in the next section, to learn the topic vectors from a corpus of unlabeled documents.

\medskip
\noindent\textbf{The Setting:}~
We make the following assumptions. 
The documents are not noisy, that is for any document $(\vec x^1, \vec x^2)$ and for all $i\in[k]$,  $\vec x^1\cdot \vec v_i  = \vec x^2\cdot\vec v_i$.
Regarding distribution $\P$, we assume that a non-negligible probability density is assigned to pure samples for each class. More formally, for some $\xi > 0$, for all $i\in[k]$, $\Pr_{\vec w \sim \P}[ \vec w = \vec e_i] \geq \xi$. 
Regarding distribution $\D^{\vec w}$, we allow the two paragraphs in a document, i.e., the two views $(\vec x^1, \vec x^2)$ drawn from $\D^{\vec w}$, to be correlated as long as 
for any subspace $Z \subset \nul\{ \vec v_1 \dots, \vec v_k \}$ of dimension strictly less than $n-k$, $\Pr_{(\vec x^1, \vec x^2)\sim \D^{\vec w}} [(\vec x^1 - \vec x^2) \not\in Z] \geq \zeta$ for some non-negligible $\zeta$.
One way to view this in the context of topic modeling is that if, say, ``sports'' is a topic, then it should not be the case that the second paragraph always talks about the exact same sport as the first paragraph; else ``sports'' would really be a union of several separate but closely-related topics.  Thus, while we do not require independence we do require some non-correlation between the paragraphs.

\medskip
\noindent\textbf{Algorithm and Analysis:}~
The main idea behind our approach is to use the consistency of the two views of the samples to first recover the subspace spanned by $\vec v_1, \dots, \vec v_k$ (Phase 1). Once this subspace is recovered, we show that a projection of a sample on this space corresponds to the convex combination of class vectors using the appropriate mixture weight that was used for that sample. Therefore, we find vectors $\vec a_1, \dots, \vec a_k$ that purely belong to each class by taking the extreme points of the projected samples (Phase 2). The class vectors $\vec v_1, \dots, \vec v_k$ are the unique vectors (up to permutations) that classify $\vec a_1, \dots, \vec a_k$ as pure samples.  Phase 2 is similar to that of \cite{arora2012learning}.
Algorithm~\ref{alg:noisefree} formalizes the details of this approach.

\begin{algorithm}
\caption {\textsc{Algorithm for Generalized Topic Models ---  No noise}} 
\label{alg:noisefree}
\textbf{Input:} A sample set $S = \{ (\vec x_i^1 , \vec x_i^2) \mid i\in[m] \}$ such that for each $i$, first a vector $\vec w$ is drawn from $\P$ and then $(\vec x_i^1, \vec x_i^2)$ is drawn from $\D^{\vec w}$.\\
\textbf{Phase 1:}
\begin{enumerate}
\item Let $X^1$ and $X^2$ be matrices where the  $i^{th}$ column is $\vec x^1_i$ and $\vec x^2_i$, respectively.
\item Let $P$ be the projection matrix on the last $k$ left singular vectors of $(X^1 - X^2)$.
\end{enumerate}
\textbf{Phase 2:}
\begin{enumerate}
\item Let $S_\parallel = \{P\vec x_i^j \mid i\in[m], j\in\{1,2\} \}$.
\item Let $A$ be a matrix whose columns are the extreme points of the convex hull of $S_\parallel$. (This can be found using farthest traversal or linear programming.)
\end{enumerate}
\textbf{Output:} Return columns of $A^+$ as $\vec v_1, \dots, \vec v_k$.
\end{algorithm}

In Phase $1$ for recovering  $\spn\{\vec v_1, \dots, \vec v_k\}$, note that for any sample $(\vec x^1, \vec x^2)$ drawn from $\D^{\vec w}$, we have that $\vec v_i\cdot \vec x^1 = \vec v_i\cdot \vec x^2= w_i$. Therefore, regardless of what $\vec w$ was used to produce the sample, we have that $\vec v_i\cdot (\vec x^1 - \vec x^2) = 0$ for all $i\in [k]$.
That is, $\vec v_1, \dots, \vec v_k$ are in the null-space of all such $(\vec x^1 - \vec x^2)$. So, if  samples $(\vec x_i^1 - \vec x_i^2)$ span a $n-k$ dimensional subspace, then $\spn\{\vec v_1, \dots, \vec v_k\}$  can be recovered by taking $\nul\{ (\vec x^1 - \vec  x^2) \mid (\vec x^1, \vec x^2) \in \X^{\vec w} \times \X^{\vec w},~ \forall \vec w\in \R^k\}$.
Using singular value decomposition, this null space is spanned by the last  $k$ singular vectors of $X^1 - X^2$, where $X^1$ and $X^2$ are matrices with columns $\vec x_i^1$ and $\vec x_i^2$, respectively.

This is where the assumptions on $\D^{\vec w}$ come into play. By assumption, for any strict subspace $Z$ of $\spn\{ (\vec x^1 - \vec  x^2) \mid (\vec x^1, \vec x^2) \in \X^{\vec w} \times \X^{\vec w},~ \forall \vec w\in \R^k\}$, 
$D^{\vec w}$ has non-negligible probability on instances $(\vec x^1 - \vec x^2) \notin Z$. Therefore, after seeing sufficiently many samples  we can recover the space of all $(\vec x^1 - \vec  x^2)$. The next lemma, whose proof appears in Appendix~\ref{app:no-noise-rank}, formalizes this discussion.

\begin{lemma}\label{lem:rank}
Let $Z = \spn \{(\vec x^1_i - \vec x^2_i)\mid i \in[m]\}$. Then,  $m = O( \frac {n-k}{\zeta} \log(\frac 1 \delta))$ is sufficient such that with probability $1-\delta$, $\rank(Z) = n-k$.
\end{lemma}
 \begin{wrapfigure}[14]{r}{0.25\textwidth}
 \vspace{-1cm}
  \begin{center}
    \includegraphics[width=0.2\textwidth]{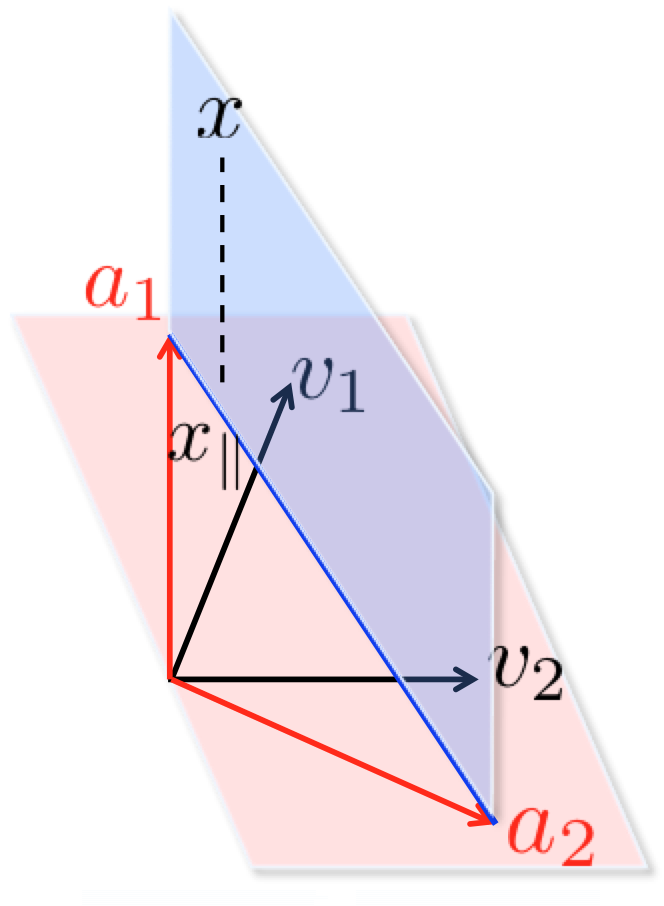}
     \vspace{-0.5cm}
  \end{center}
\caption{\small $\vec v_1, \vec v_2$ correspond to class $1$ and $2$, and $\vec a_1$ and $\vec a_2$ correspond to canonical vectors that are purely of class $1$ and $2$, respectively.}
\label{fig:w-and-u}
\end{wrapfigure}
Using Lemma~\ref{lem:rank}, Phase 1 of Algorithm~\ref{alg:noisefree} recovers $\spn\{\vec v_1, \dots, \vec v_k\}$. Next, we show that pure samples are the extreme points of the convex hull of all samples when projected on the subspace $\spn \{\vec v_1, \dots, \vec v_k\}$. 
Figure~\ref{fig:w-and-u} demonstrates the relation between the class vectors, $\vec v_i$, projection of samples, and the projection of pure samples $\vec a_i$. The next lemma,
 whose proof appears in Appendix~\ref{app:no-noise-sum-alpha-u}, formalizes this claim.

\begin{lemma} \label{lem:sum-alpha-u}
For any $\vec x$, let $\vec x_\parallel$ represent the projection of $\vec x$ on $\spn\{\vec v_1, \dots, \vec v_k\}$. Then, $x_\parallel = \sum_{i\in[k]}(\vec v_i \cdot \vec x)  \vec a_i.$
\end{lemma}
With $\sum_{i\in[k]}(\vec v_i \cdot \vec x) \vec a_i$ representing the projection of $\vec x$ on  $\spn\{\vec v_1, \dots, \vec v_k\}$, it is clear that the extreme points of the set of all projected instances that belong to $\X^{\vec w}$ for all ${\vec w}$ are $\vec a_1, \dots, \vec a_k$. Since in a large enough sample set, with high probability for all $i\in[k]$, there is a pure sample of type $i$, taking the extreme points of the set of projected samples is also $\vec a_1, \dots, \vec a_k$. The following lemma, whose proof appears in Appendix~\ref{app:no-noise-extreme-no-noise}, formalizes this discussion.

\begin{lemma} \label{lem:extreme-no-noise}
Let $m = c(\frac 1\xi \log(\frac k \delta))$ for a large enough constant $c>0$.
Let $P$ be the projection matrix for $\spn\{\vec v_1, \dots, \vec v_k\}$ and $S_\parallel = \{P\vec x_i^j \mid i\in[m], j\in\{1,2\} \}$ be the set of projected samples.
With probability $1-\delta$, 
$\{ \vec a_1, \dots, \vec a_k\}$ is the set of extreme points of $\convh(S_\parallel)$.
\end{lemma}

Therefore, $\vec a_1, \dots, \vec a_k$ can be learned by taking the extreme points of the convex hull of all samples projected on $\spn(\{\vec v_1, \dots, \vec v_k\})$. Furthermore, $V = A^+$ is unique, therefore $\vec v_1, \dots, \vec v_k$ can be easily found by taking the pseudoinverse of matrix $A$. Together with Lemma~\ref{lem:rank} and \ref{lem:extreme-no-noise} this proves the next theorem regarding learning class vectors in the absence of noise.

\begin{theorem}[No Noise] \label{thm:no-noise}
There is a polynomial time algorithm for which
$ m = O\left( \frac {n-k}{\zeta} \ln(\frac 1 \delta) + \frac 1\xi \ln(\frac k \delta) \right)
$
is sufficient to recover  $\vec v_i$   exactly for all $i\in[k]$, with probability $1-\delta$.
\end{theorem}

\section{Relaxing the Assumptions} \label{sec:noise}
In this section, we relax the two main simplifying assumptions from Section~\ref{sec:no-noise}. 
We relax the assumption on non-noisy documents and allow a  large fraction of the documents to not satisfy  $\vec v_i \cdot \vec x^1=  \vec v_i \cdot \vec x^2$. In the standard topic model, this corresponds to having a large fraction of short documents.
Furthermore, we relax the assumption on the existence of pure documents to an assumption on the existence of ``almost-pure'' documents. 
We further develop the approach discussed in the previous section and  introduce  efficient algorithms that approximately recover the topic vectors in this setting.

\noindent\textbf{The Setting:}~
We assume that any sampled document  has a non-negligible probability of being non-noisy and with the remaining probability, 
the two views of the document are perturbed by  additive Gaussian noise, independently. More formally, for a given sample $(\vec x^1, \vec x^2)$, with probability $p_0>0$ the algorithm receives $(\vec x^1, \vec x^2)$ and with the remaining probability $1-p_0$, the algorithm receives 
 $(\hvec{x}^1, \hvec{x}^2)$, such that $\hat{\vec x}^j =  {\vec x}^j  + \vec e^j$, where $\vec e^j \sim \N(\vec 0, \sigma^2 I_n)$.
 
We assume that for each topic the probability that a document is mostly about that topic is non-negligible. More formally, for any topic $i\in[k]$, $\Pr_{\vec w\sim \P} [ \| \vec e_i - \vec w\|_1 \leq \epsilon \|] \geq g(\epsilon)$, where $g$ is a polynomial function of its input. 
A stronger form of this assumption, better known as the \emph{dominant admixture assumption}, assumes that every document is mostly about one topic and has been empirically shown to hold on several real world data sets~\citep{bansal2014provable}.
Furthermore, in the  Latent Dirichlet Allocation model, $\Pr_{\vec w\sim \P} [\max_{i\in[k]}  w_i \geq 1-\epsilon] \geq O(\epsilon^2)$ for typical values of the  concentration parameter.

We also make mild assumptions on the distribution over instances.
We assume that the covariance of the distribution over $(\vec x_i^1 - \vec x_i^2) (\vec x_i^1 - \vec x_i^2)^\top$ is significantly larger than the noise covariance $\sigma^2$. That is, for some $\delta_0>0$, the least significant non-zero eigen value of $\E_{(\vec x_i^1 , \vec x_i^2)}  [ (\vec x_i^1 - \vec x_i^2) (\vec x_i^1 - \vec x_i^2)^\top ]$, equivalently its $(n-k)^{th}$ eigen value,  is greater than $6 \sigma^2 + \delta_0$.
At a high level, these assumptions are necessary, because if $\| \vec x_i^1 - \vec x_i^2 \|$ is too small
compared to $\|\vec x_i^1\|$ and $\| \vec x_i^2\|$, 
 then even a small amount of noise affects the structure present in  $\vec x_i^1 - \vec x_i^2$ completely. Moreover, we assume that the $L_2$ norm of each view of a sample is bounded by some $M>0$.
We also assume that for all $i\in [k]$, $\| \vec a_i  \| \leq \sizeA$ for some $\sizeA>0$. At a high level,  $\| \vec a_i \|$s are inversely proportional to the non-zero singular values of $V = (\vec v_1, \dots, \vec v_k)$. Therefore, $\| \vec a_i\| \leq \sizeA$ implies that the $k$ topic vectors are sufficiently different.

\medskip
\noindent\textbf{Algorithm and Results:}~
Our approach follows the general theme of the previous section:  First, recover 
 $\spn\{\vec v_1, \dots, \vec v_k\}$ and then recover $\vec a_1, \dots, \vec a_k$ by taking the extreme points of the projected samples. 
In this case, in the first phase we recover $\spn\{\vec v_1, \dots, \vec v_k\}$ approximately, by finding a projection matrix $\hat P$ such that $\|P - \hat P\|\leq \epsilon$ for an arbitrarily small $\epsilon$, where $P$ is the projection matrix on $\spn\{\vec v_1, \dots, \vec v_k\}$.
At this point in the algorithm, the projection of samples on $\hat P$ can include points that are arbitrarily far from $\Delta$. This is due to the fact that the noisy samples are  perturbed by $\N(\vec 0, \sigma^2I_n)$, so, for large values of $\sigma$ some noisy samples  map to points that are quite far from $\Delta$. Therefore, we have to detect and remove these samples before continuing to the second phase. 
For this purpose, we show that the low density regions of the projected samples can safely be removed such that the convex hull of the remaining points is close to $\Delta$.
 In the second phase, we consider projections of each sample using $\hat P$.
To approximately recover $\vec a_1, \dots, \vec a_k$, we recover samples, $\vec x$, that are far from the convex hull of the remaining points, when  $\vec x$ and a ball of points close to it are removed. We then show that such points are close to one of the pure class vectors,  $\vec a_i$. Algorithm~\ref{alg:noise-gaussian} and the details of the above approach  and its performance are as follows.

\begin{algorithm}[ht]
\caption {\textsc{Algorithm for Generalized Topic Models ---  With Noise}} 
\label{alg:noise-gaussian}
\textbf{Input:} A sample set $\{ ({\hat{\vec x}}_i^1 , {\hat{\vec x}}_i^2) \mid i \in[m] \}$ such that for each $i$, first a vector $\vec w$ is drawn from $\P$, then $(\vec x_i^1, \vec x_i^2)$ is drawn from $\D^{\vec w}$, then with probability $p_0$, $\hvec x_i^j  = \vec x_i^j$, else with probability $1-p_0$, ${\hat{\vec x}}_i^j = \vec x_i^j + \N(\vec 0, \sigma^2 I_n)$ for $i\in [m]$ and $j\in \{1,2\}$.
\\
\textbf{Phase 1:}
\begin{enumerate}
\item Take $m_1 = \Omega \left(\frac {n-k}{\zeta} \ln(\frac 1 \delta)
+ \frac{n \sigma^4 r^2 M^2}{\delta_0^2 \epsilon^2} \ln(\frac 1 \delta)
+  \frac{n \sigma^2 M^4 r^2}{\delta_0^2 \epsilon^2} \polylog(\frac{nrM}{\epsilon\delta})
+ \frac{M^4}{\delta_0^2} \ln(\frac n \delta)
 \right)$ samples.
\item Let $\hat X^1$ and $\hat X^2$ be matrices where the  $i^{th}$ column is $\hvec x^1_i$ and $\hvec x^2_i$, respectively.
\item Let $\hat P$ be the projection matrix on the last $k$  left singular vectors of $\hat X^1 - \hat X^2$.
\end{enumerate}
\textbf{Denoising Phase:}
\begin{enumerate} \setcounter{enumi}{3}
\item Let  $\epsilon' = \frac{\epsilon}{8 r}$ and $\pg = g\left(\frac {\epsilon'}{8 k\sizeA} \right)$.
\item 
Take $m_2 = \Omega\left( \frac{k}{p_0\pg} \ln \frac1\delta \right)$ fresh samples\footnotemark\ and let $\hat S_\parallel= \left\{\hat P\hvec x_i^1 \mid \forall i\in[m_2] \right\}$.

\item Remove $\hvec x_\parallel$ from $\hat S_\parallel$, for which there are less than $p_0 \pg m_2/2$ points within distance of $\frac{\epsilon'}{2}$  in $\hat S_\parallel$. \label{item:S||}
\end{enumerate}
\textbf{Phase 2:}
\begin{enumerate} \setcounter{enumi}{5}
	\item For all $\hvec x_\parallel$ in $\hat S_\parallel$, if $
\dist(\vec x_\parallel,  \convh(\hat S_\parallel \setminus B_{6r\epsilon'}(\hvec x) ) \geq  2\epsilon'$ add $\hvec x_\parallel$ to $C$.
	\item Cluster $C$ using single linkage with threshold $16r\epsilon'$. Assign any point from cluster $i$ as $\hat {\vec a}_i$. \end{enumerate}
\textbf{Output:}
Return $\hvec a_1, \dots, \hvec a_k$.
\end{algorithm}
\footnotetext{For the denoising step, we use a fresh set of samples that were not used for learning the projection matrix. This guarantees that the noise distribution  in the projected samples remain a Gaussian.}

\begin{theorem}\label{thm:noise-a_i}
Consider any $\epsilon, \delta>0$ such that
$\epsilon \leq  O\left(r \sigma \sqrt{ k} \right)$,
where  $r$ is a parameter that depends on the geometry of the simplex $\vec a_1, \dots, \vec a_k$ and will be defined later.
There is an efficient algorithm for which an  unlabeled sample set of size
\[
m = O\left(\frac {n-k}{\zeta} \ln(\frac 1 \delta)
+ \frac{n \sigma^4 r^2 M^2}{\delta_0^2 \epsilon^2} \ln(\frac 1 \delta)
+  \frac{n \sigma^2 M^4 r^2}{\delta_0^2 \epsilon^2} \polylog(\frac{nrM}{\epsilon\delta})
+ \frac{M^4}{\delta_0^2} \ln(\frac n \delta)  +\frac{k~\ln(1 / \delta)}{p_0 ~g\!\left(\epsilon / (k r \sizeA)\right)} 
 \right)
\]
is sufficient  to recover $\hvec a_i$  such that $\| \hvec a_i - \vec a_i\|_2 \leq \epsilon$  for all $i\in[k]$,  with probability $1-\delta$.
\end{theorem}

The proof of Theorem~\ref{thm:noise-a_i} involves the next  three lemmas on the performance of the phases of the above algorithm. We formally state these two lemmas here, but defer their proofs to Sections~\ref{sec:phase1}, \ref{sec:denoise} and \ref{sec:phase2}.

\begin{lemma}[Phase 1]\label{lem:phase1-noise}
For any  $\sigma>0$ and $\epsilon>0$, an unlabeled sample set of size
\[
m= O \left(\frac {n-k}{\zeta} \ln(\frac 1 \delta)
+ \frac{n \sigma^4}{\delta_0^2 \epsilon^2} \ln(\frac 1 \delta)
+  \frac{n \sigma^2 M^2}{\delta_0^2 \epsilon^2} \polylog(\frac{n}{\epsilon\delta})
+ \frac{M^4}{\delta_0^2} \ln(\frac n \delta)
 \right).
\]
 is sufficient, such that  with probability $1-\delta$, Phase 1 of Algorithm~\ref{alg:noise-gaussian} returns a projection matrix $\hat P$, such that $\| P - \hat P\|_2 \leq \epsilon$.

\end{lemma}

\begin{lemma}[Denoising]\label{lem:phase-denoise}
Let $\epsilon' \leq  \frac 13\sigma \sqrt{k}$,
$\|P - \hat P \| \leq \epsilon'/8 M$, and $\pg = g\left(\frac {\epsilon'}{8 k\sizeA} \right)$.
An unlabeled sample size of $m = O\left( \frac{k}{p_0 \pg} \ln(\frac1\delta) \right)$ is sufficient such that for
$\hat S_\parallel$ defined in Step~\ref{item:S||} of Algorithm~\ref{alg:noise-gaussian} the following holds with probability $1-\delta$: For any $\vec x \in \hat S_\parallel$, $\dist(\vec x, \Delta)\leq \epsilon'$, and, for all $i\in [k]$, there exists $\hvec a_i \in \hat S_\parallel$ such that $\| \hvec a_i - \vec a_i\|\leq \epsilon'$.
\end{lemma}

\begin{lemma}[Phase 2]\label{lem:phase2-noise}
Let $\hat S_\parallel$ be a set of points for which the conclusion of Lemma~\ref{lem:phase-denoise} holds with the value of $\epsilon' = \epsilon / 8r$.
Then, Phase 2 of Algorithm~\ref{alg:noise-gaussian} returns $\hvec a_1, \dots, \hvec a_k$ such that for all $i\in [k]$, $\| \vec a_i - \hvec a_i \| \leq \epsilon$. 
\end{lemma}

We now prove our main Theorem~\ref{thm:noise-a_i} by directly leveraging the three lemmas  we just stated.
\begin{proof}[Proof of Theorem~\ref{thm:noise-a_i}]
By Lemma~\ref{lem:phase1-noise}, sample set of size  $m_1$ is sufficient such that Phase 1 of  Algorithm~\ref{alg:noise-gaussian} leads to $\|  P - \hat P \| \leq \frac{\epsilon}{32 M r}$, with probability $1-\delta/2$.
Let $\epsilon' = \frac {\epsilon}{8r}$ and take a fresh sample of size $m_2$. 
By Lemma~\ref{lem:phase-denoise}, with probability $1-\delta/2$, for any $\vec x \in \hat S_\parallel$, $\dist(\vec x, \Delta)\leq \epsilon'$, and, for all $i\in [k]$, there exists $\hvec a_i \in \hat S_\parallel$ such that $\| \hvec a_i - \vec a_i\|\leq \epsilon'$. 
Finally, applying Lemma~\ref{lem:phase2-noise} we have that Phase 2 of Algorithm~\ref{alg:noise-gaussian} returns $\hvec a_i$, such that for all $i\in[k]$, $\| \vec a_i - \hvec a_i\| \leq\epsilon$.
\end{proof}
Theorem~\ref{thm:noise-a_i} discusses the approximation  of $\vec a_i$ for all $i\in[k]$. It is not hard to see that such an approximation also translates to the approximation of class vectors, $\vec v_i$ for all $i\in[k]$. That is, using the properties of perturbation of pseudoinverse matrices (see Proposition~\ref{prop:inverseperturb}) one can show that $\| \hat A^{+} - V \| \leq O(\| \hat A - A\| )$. Therefore, $\hat V = \hat A^+$  is a good approximation for $V$.

\subsection{Proof of Lemma~\ref{lem:phase1-noise} --- Phase 1}   \label{sec:phase1}
For $j\in \{1, 2\}$, let $X^j$ and $\hat X^j$ be $n\times m$ matrices with the $i^{th}$ column being  $\vec x^j_i$ and $\hvec x^j_i$, respectively.
As we demonstrated in Lemma~\ref{lem:rank}, with high probability $\rank(X^1 - X^2) = n-k$. 
Note that the nullspace of columns of $X^1  - X^2$ is spanned by the left singular vectors of  $X^1  - X^2$ that correspond to its $k$ zero singular values. Similarly, consider the space spanned by the $k$ least left singular vectors of  $\hat X^1  - \hat X^2$.
We show that the nullspace of columns of $X^1  - X^2$ can be approximated within any desirable accuracy by the space spanned by the $k$ least left singular vectors of $\hat X^1  - \hat X^2$, given a sufficiently large number of samples.

Let $D = X^1 - X^2$ and $\hat D = \hat X^1 - \hat X^2$. 
For ease of exposition, assume that all samples are perturbed by Gaussian noise $\N(\vec 0,\sigma^2 I_n)$.\footnote{The assumption that with a non-negligible probability a sample is non-noisy is not needed for the analysis and correctness of Phase 1 of Algorithm~\ref{alg:noise-gaussian}. This assumption only comes into play in the denoising phase.}
Since each view of a sample is perturbed by an independent draw from a Gaussian noise distribution, we can view $\hat D = D + E$, where each column of $E$ is drawn i.i.d from distribution $\N(\vec 0, 2 \sigma^2 I_n)$. Then,
$\frac 1m \hat D\hat D^\top  = \frac 1m  D D^\top  + \frac 1m D E^\top  + \frac 1m ED^\top  +\frac 1m   EE^\top.$
As a thought experiment, consider this equation in expectation. Since $\E[\frac 1m  EE^\top] = 2 \sigma^2 I_n$ is the covariance matrix of the noise and $\E[DE^\top  + E D^\top ] = 0$, we have
\begin{equation} \label{eq:shrinkage}
   \frac 1m \E \left[ \hat D \hat D^\top \right]  - 2\sigma^2 I_n= \frac 1m\E\left[D D^\top \right].
\end{equation}
Moreover, the eigen vectors and their order are the same in $\frac 1m\E[\hat D \hat D^\top ]$ and $\frac 1m\E[\hat D \hat D^\top ] -2 \sigma^2 I_n$. Therefore, one can recover the nullspace of  $\frac 1m\E[D D^\top]$ by taking the space of the least $k$ eigen vectors of $\frac 1m\E[\hat D \hat D^\top] $.
Next, we show how to recover the nullspace using $\hat D \hat D^\top$, rather than $\E[\hat D \hat D^\top]$.
Assume that the following properties hold:
\begin{enumerate}
\item Equation~\ref{eq:shrinkage} holds not only in expectation, but also with high probability.
That is, with high probability, 
$\| \frac 1m \hat D \hat D^\top   - 2\sigma^2 I_n   - \frac 1m D D^\top \|_2 \leq \epsilon.$
\item With high probability  $\lambda_{n-k} (\frac 1m \hat D \hat D^\top) > 4\sigma^2 + \delta_0 / 2$, where $\lambda_i(\cdot)$ denotes the $i^{th}$ most significant eigen value.
\end{enumerate}
Let $D = U \Sigma V^\top$ and $\hat D = \hat U \hat \Sigma \hat V^\top$ be SVD representations.
We have that  $\frac 1m \hat D \hat D ^\top - 2\sigma^2 I_n = \hat U (\frac 1m \hat \Sigma^2  - 2\sigma^2 I_n)  \hat U^\top$. By property 2,  $\lambda_{n-k}(\frac 1m\hat \Sigma^2) > 4 \sigma^2 +\delta_0/2$. That is, the eigen vectors and their order are the same in $\frac 1m \hat D\hat D^\top  - 2 \sigma^2 I_n$ and $\frac 1m \hat D\hat D^\top  $. As a result the projection matrix, $\hat P$, on the least $k$ eigen vectors of $\frac 1m \hat D \hat D^\top$, is the same as the projection matrix, $Q$, on the least $k$ eigen vectors of $\frac 1m \hat D \hat D^\top  - 2 \sigma^2 I_n$.

Recall that $\hat P$ and $P$ and $Q$ are the  projection matrices on the  least significant $k$ eigen   vectors of $\frac 1m \hat D \hat D^\top$, $\frac 1m D D^\top$,  and $\frac 1m \hat D \hat D^\top- 2\sigma^2I$, respectively. As we discussed, $\hat P = Q$. Now, using the  \cite{davis1970rotation}  or \cite{wedin1972perturbation}  $\sin\theta$ theorem (see Proposition~\ref{prop:kahan}) from
matrix perturbation theory, we have,
\begin{align*}
 \| P - \hat P \|_2 = \| P - Q \| \leq \frac{\|\frac 1m \hat D \hat D^\top - 2\sigma^2 I_n - \frac 1m D D^\top\|_2  }{ \left|  \lambda_{n-k}(\frac 1m \hat D \hat D^\top) - 2\sigma^2 -  \lambda_{n-k+1}(\frac 1m D  D^\top )\right|} \leq  \frac{2\epsilon}{\delta_0}
\end{align*}
where we use Properties 1 and 2 and the fact that $\lambda_{n-k+1}(\frac 1m D  D^\top ) =0$, in the last transition.

\subsubsection{Concentration}
It remains to prove Properties 1 and 2. We briefly describe 
our approach for obtaining concentration results and prove that  when the number of samples $m$ is large enough, with high probability $\| \frac 1m \hat D \hat D^\top   - 2\sigma^2 I_n  - \frac 1m D D^\top \|_2 \leq \epsilon$ and  $\lambda_{n-k} (\frac 1m \hat D \hat D^\top) > 4\sigma^2 + \delta_0 / 2$.

Let us first describe $\frac 1m \hat D \hat D^\top   - 2 \sigma^2 I_n  - \frac 1m D D^\top$ in terms of the error matrices. We have
\begin{equation} \label{eq:hatDexpansion}
\frac 1m \hat D \hat D^\top   - 2\sigma^2 I_n  - \frac 1m D D^\top = \left( \frac 1m E E^\top -2 \sigma^2 I_n \right) + \left( \frac 1m D E^\top + \frac 1m E D^\top \right).
\end{equation}
It suffices to show that for large enough  $m > m_{\epsilon, \delta}$, $\Pr[ \| \frac 1m E E^\top - 2\sigma^2 I_n \|_2 \geq \epsilon] \leq \delta$ and $\Pr[ \|  \frac 1m D E^\top + \frac 1m E D^\top     \|_2 \geq \epsilon] \leq \delta$. In the former, note that $\frac 1m E E^\top$ is the sample covariance of the Gaussian noise matrix and $2\sigma^2 I_n$ is the true covariance matrix of the noise distribution.
The next claim is a direct consequence of the convergence properties of sample covariance of the Gaussian distribution (see Proposition~\ref{prop:covariance-gauss}).

\begin{claim} \label{claim:EE-estimate}
For $m > n \frac{\sigma^4}{\epsilon^2} \log(\frac 1 \delta)$,  with probability $1-\delta$, $ \| \frac 1m E E^\top - 2\sigma^2 I_n \|_2 \leq \epsilon$. \footnote{At first sight, the dependence of this sample complexity on $\sigma$ might appear unintuitive. But, note that even without seeing any samples we can approximate the noise covariance within $2\sigma^2 I_n$. Therefore, if $\epsilon = 2\sigma^2$ our work is done.}
\end{claim}

We use the Matrix  Bernstein inequality~\citep{tropp2015introduction}, described  in Appendix~\ref{app:spectral}, to demonstrate the   concentration of  $ \| \frac 1m D E^\top + \frac 1m E D^\top\|_2$. The proof of the next Claim is relegated to Appendix~\ref{app:claim:DE-estaimte}.

\begin{claim} \label{claim:DE-estimate}
$m = O(\frac{n \sigma^2 M^2}{\epsilon^2} \polylog\frac{n}{\epsilon\delta})$ is sufficient so that with probability $1-\delta$,  $\left\|  \frac 1m D E^\top + \frac 1m ED^\top   \right\|_2 \leq \epsilon$,
\end{claim}

Next, we prove that $\lambda_{n-k} (\frac 1m \hat D \hat D^\top) > 4\sigma^2 + \delta_0 / 2$. Since for any two matrices, the difference in $\lambda_{n-k}$ can be bounded by the spectral norm of their difference (see Proposition~\ref{prop:diffeigen}), using Equation~\ref{eq:hatDexpansion}, we have
{\small
\begin{align*}
\left| \lambda_{n-k}\left(\frac 1m \hat D \hat D^\top\right)   - \lambda_{n-k}\left(\frac 1m D  D^\top\right)   \right|
 \leq \left\| 2 \sigma^2 I + \left( \frac 1m E E^\top - 2\sigma^2 I_n \right) - \left( \frac 1m D E^\top + \frac 1m E D^\top\right) \right\| 
 \leq 2 \sigma^2 + \frac{ \delta_0}{4},
\end{align*}}where in the last transition we use Claims~\ref{claim:EE-estimate} and \ref{claim:DE-estimate} with the value of $\delta_0/8$ to bound the last two terms  by a total of $\delta_0/4$.
Since $\lambda_{n-k}(\E[\frac 1m D  D^\top]) \geq 6 \sigma^2 + \delta_0$, it is sufficient to show that $|\lambda_{n-k}(\E[\frac 1m D  D^\top]) - \lambda_{n-k}([\frac 1m D  D^\top]) | \leq \delta_0/4$. Similarly as before, this is bounded by $\|  \frac 1m D  D^\top - \E[\frac 1m D  D^\top] \|$. We use the Matrix Bernstein inequality (Proposition~\ref{prop:bernstein}) to prove this concentration result.   The rigorous proof of this claim appears in Appendix~\ref{app:estimateDD},
\begin{claim}\label{claim:DD-estimate}
$m = O\left( \frac{M^4}{\delta_0^2} \log \frac n \delta \right)$ is sufficient so that with probability $1-\delta$,
 $\left\|  \frac 1m D D^\top  -  \E\left[ \frac 1m D D^\top  \right] \right\|_2 \leq \frac {\delta_0}{4}$.
\end{claim}
This completes the analysis of Phase 1 of our algorithm and the proof of Lemma~\ref{lem:phase1-noise} follows directly from the above analysis and the application of Claims~\ref{claim:EE-estimate} and \ref{claim:DE-estimate} with the error of $\epsilon \delta_0$, and Claim~\ref{claim:DD-estimate}.

\subsection{Proof of Lemma~\ref{lem:phase-denoise} --- Denoising Step} \label{sec:denoise}

Having approximately recovered a projection matrix $\hat P$ for $\spn \{\vec v_1, \dots, \vec v_k \}$, we can now use this subspace to partially  denoise the samples while approximately preserving $\Delta = \convh(\{\vec a_1, \dots, \vec a_k\})$.
At a high level, when considering the projection of samples on $\hat P$, one can show that 1)
the regions around $\vec a_i$ have sufficiently high density, and, 2) the regions that are far from $\Delta$ have low density.

We claim that if $\hat x_\parallel \in \hat S_\parallel$ is \emph{non-noisy and corresponds almost purely to one class} then $\hat S_\parallel$ also includes a non-negligible number of points within $O(\epsilon')$ distance of $\hat x_\parallel$.
This is due to the fact that a non-negligible number of points (about $p_0 \pg m$ points) correspond to non-noisy and almost-pure samples that using $P$ would get projected  to points within a distance of $O(\epsilon')$ of each other.
Furthermore, the inaccuracy in $\hat P$ can only perturb the projections up to $O(\epsilon')$ distance. So, the projections of all non-noisy samples that are purely of class $i$ fall within $O(\epsilon')$ of $\vec a_i$. The following lemma, whose proof appears in Appendix~\ref{app:claim:high-density}, formalizes this claim.

In the following lemmas, let $D$ denote the  flattened distribution of the first paragraphs. That is, the distribution over $\hvec x^1$ where we first take $\vec w\sim \P$, then take $(\vec x^1, \vec x^2) \sim \D^{\vec w}$, and finally take $\hvec x^1$.

\begin{claim} \label{claim:high-dense}
For all $i\in [k]$,
$\Pr_{\vec x \sim D}\left[ \hat P\vec x \in B_{\epsilon'/4}(\vec a_i)\right] \geq p_0  \pg.$
\end{claim}

On the other hand, any projected point that is far from the convex hull of $\vec a_1, \dots, \vec a_k$ has to be noisy, and as a result, has been generated by a Gaussian distribution with variance $\sigma^2$. For a choice of $\epsilon'$ that is small with respect to $\sigma$, such points do not concentrate well within any ball of radius $\epsilon'$. 
In the next lemma we show that the regions that are far from the convex hull have low density.

\begin{claim} \label{claim:low-dense}
For any $\vec z$ such that $\dist(\vec z, \Delta) \geq \epsilon'$, we have
$\Pr_{\vec x \sim D}\left[ \hat P \vec x \in B_{\epsilon'/2}(\vec z)\right] \leq \frac{p_0 \pg}{4}.
$
\end{claim}
\noindent{\it Proof.}~
We first show that $B_{\epsilon'/2}(\vec z)$ does not include any non-noisy points. Take any non-noisy sample $\vec x$. Note that $P\vec x = \sum_{i=1}^k w_i \vec a_i$, where $w_i$ are the mixture weights corresponding to point $\vec x$. We have,
\[  \left\|  \vec z - \hat P \vec x \right\| =   \left\| \vec z - \sum_{i=1}^k w_i \vec a_i + (P - \hat  P) \vec x \right\| \geq 
 \left\| \vec z - \sum_{i=1}^k w_i \vec a_i  \right\| - \| P -  \hat P\| \|\vec x\| \geq\epsilon'/2
\]
\begin{wrapfigure}[11]{r}{0.35\textwidth}
 \vspace{-1cm}
  \begin{center}
    \includegraphics[width=0.29\textwidth]{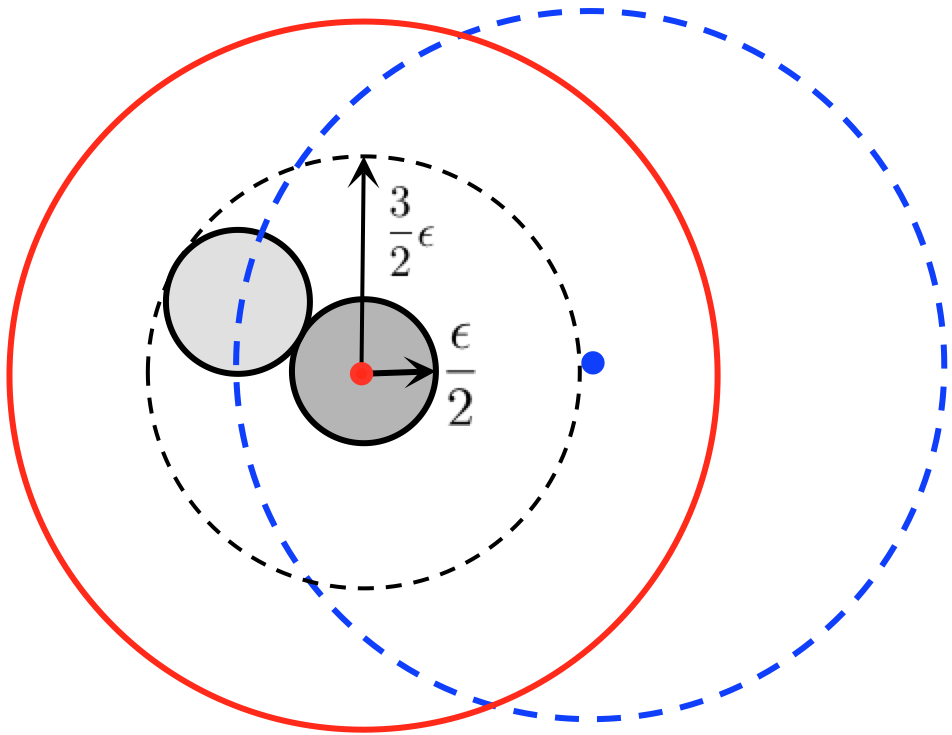}
  \end{center}
 \vspace{-0.6cm}\caption{\small Density is maximized when blue and red gaussians coincide and the ball is at their center.}
\label{fig:gauss2}
\end{wrapfigure}
Therefore, $B_{\epsilon'/2}(\vec z)$ only contains noisy points. 
Since  noisy points are perturbed by a spherical 
Gaussian, the projection of these points on any $k$-dimensional subspace can be thought of points generated from a $k$-dimensional Gaussian distributions with variance $\sigma^2$ and potentially different centers.
One can show that the densest ball of any radius is at the center of a Gaussian.
Here, we prove a slightly weaker claim. 
Consider one such Gaussian distribution, $\N(\vec 0, \sigma^2 I_k)$.  Note that the pdf of the Gaussian distribution decreases as we get farther from its center. By a coupling between the density of the points, $B_{\epsilon'/2}(\vec 0)$ has higher density than any $B_{\epsilon'/2}(\vec c)$ with $\|\vec  c\|_2>  \epsilon'$.
Therefore, 
\[ \sup_{\vec c}  \Pr_{\vec x\sim \N(\vec 0, \sigma^2I_k)}[\vec x\in B_{\epsilon'/2}(\vec c)] \leq \Pr_{\vec x\sim \N(\vec 0, \sigma^2I_k)}[\vec x \in B_{3\epsilon'/2}(\vec 0)].  
\]
So, over $D$ this value will be maximized if the Gaussians had the same center (see Figure~\ref{fig:gauss2}). Moreover, in $\N(\vec 0, \sigma^2 I_k)$, 
$\Pr[\|\vec x\|_2 \leq \sigma \sqrt{ k (1- t)}]\leq \exp(-k t^2/ 16).
$
Since
 $3 \epsilon'/2 \leq   \sigma\sqrt{k} /2   \leq  \sigma \sqrt{ k (1- \sqrt{\frac{16}{k} \ln\frac{4}{p_0 \pg}})}$ we have
\[ \Pr_{\hvec x \sim D} [\vec x\in B_{\epsilon'/2}(\vec c)] \leq     \Pr_{\vec x\sim \N(\vec 0, \sigma^2I_k)}[\|\vec x\|_2 \leq 3\epsilon'/2]          \leq  \frac{p_0 \pg}{4}.\]
\qed

The next claim shows that in a large sample set,  the fraction of samples that fall within any of the described regions in Claims~\ref{claim:high-dense} and \ref{claim:low-dense} is close to the density of that region. The proof of this claim follows from VC dimension of the set of balls.

\begin{claim} \label{claim:vc}
Let $D$  be any distribution over $\R^k$ and $\vec x_1, \dots, \vec x_m$ be $m$ points drawn i.i.d from $D$.  Then $m = O(\frac{k}{\gamma} \ln \frac{1}{\delta})$ is sufficient so that with probability $1-\delta$, for any  ball $B \subseteq \R^k$ such that 
$\Pr_{\vec x\sim D}[\vec x\in B] \geq 2\gamma$, $| \{\vec x_i \mid \vec x_i \in B\} |> \gamma m$ and for any
ball $B \subseteq \R^k$ such that 
$\Pr_{\vec x\sim D}[\vec x\in B] \leq \gamma/2$, $| \{\vec x_i \mid \vec x_i \in B\} | < \gamma m$. 

\end{claim}
Therefore, upon seeing $\Omega(\frac{k}{p_0 \pg} \ln \frac1\delta)$ samples, with probability $1-\delta$, for all $i\in[k]$ there are more than $p_0 \pg m/2$ projected points within distance $\epsilon'/4$ of $\vec a_i$ (by Claims~\ref{claim:high-dense} and \ref{claim:vc}), and, no point that is $\epsilon'$ far from $\Delta$ has more than $p_0 \pg m/2$ points in its $\epsilon'/2$-neighborhood (by Claims~\ref{claim:low-dense} and \ref{claim:vc}).
Phase 2 of Algorithm~\ref{alg:noise-gaussian} leverages these properties of the set of projected points for denoising the samples while preserving $\Delta$: Remove any point from $\hat S_\parallel$ that has fewer than $p_0 \pg m/2$  neighbors within distance $\epsilon'/2$.

We conclude the proof of Lemma~\ref{lem:phase-denoise} by noting that the remaining points in $\hat S_\parallel$ are all within distance $\epsilon'$ of $\Delta$. Furthermore, any point in $B_{\epsilon'/4}(\vec a_i)$ has more than  $p_0 \pg m/2$ points within distance of $\epsilon'/2$. Therefore, such points remain in $\hat S_\parallel$ and any one of them can serve as $\hvec a_i$ for which $\| \vec a_i - \hvec a_i\|\leq \epsilon'/4$.

\subsection{Proof of Lemma~\ref{lem:phase2-noise} --- Phase 2} \label{sec:phase2}
\begin{figure}
\centering
  \begin{subfigure}{0.55\textwidth}
  \centering
          \vspace*{-0.5cm}
    \includegraphics[width=0.85\textwidth]{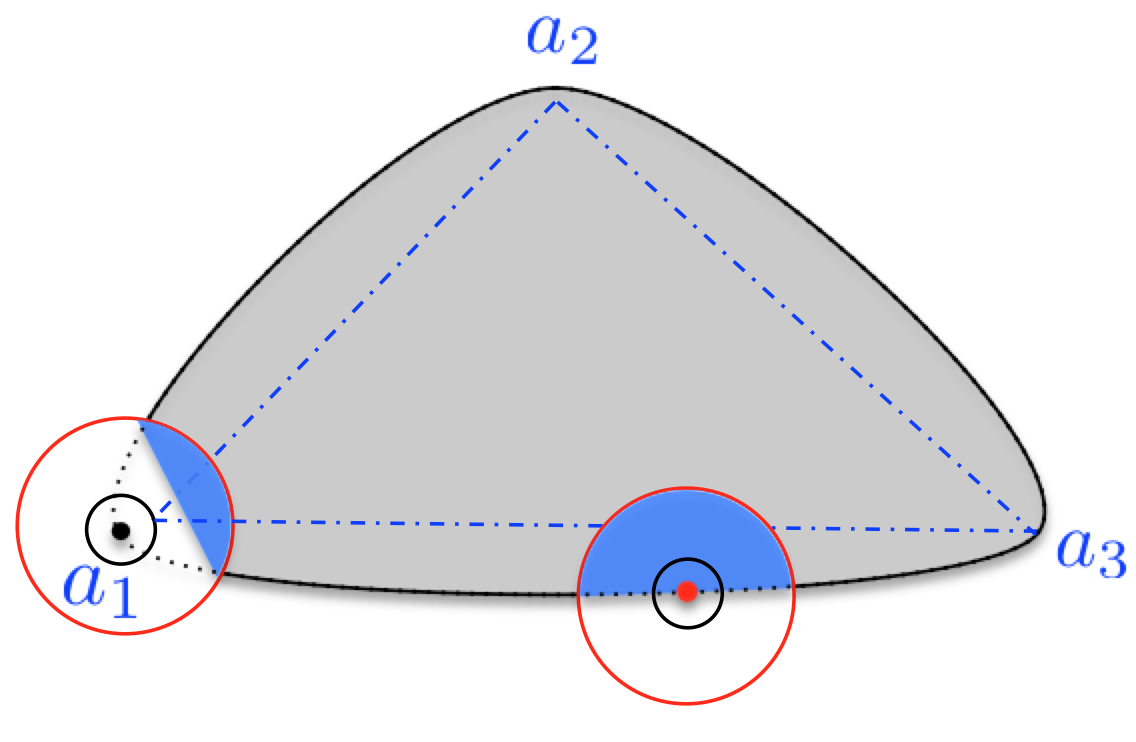}
     \caption{}
     \label{fig:noisysimplex}
  \end{subfigure}
  ~
  \begin{subfigure}{0.4 \textwidth}
  \centering
        \includegraphics[width=0.75\textwidth]{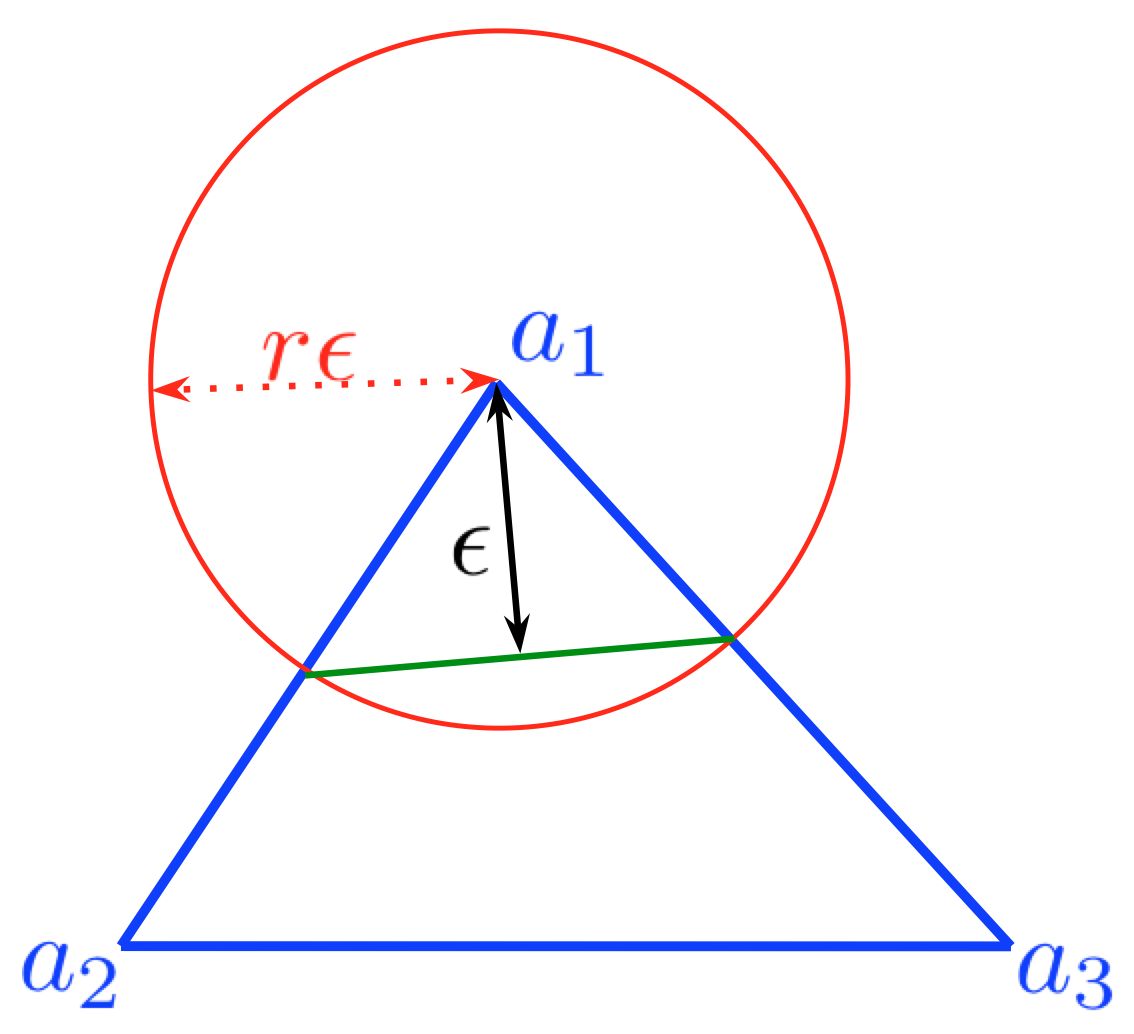}
        \caption{}
        \label{fig:skew}
  \end{subfigure}
  \caption{\small(a)~ Demonstrating the distinction between points close to $\vec a_i$  and far from $\vec a_i$. The convex hull of $CH(\hat S_{||} \setminus B_{r_2}(\hvec x))$, which is a subset of the blue and gray region, intersects $B_{r_1}(\hvec x)$ only for $\hvec x$ that is sufficiently far from $\vec a_i$'s. (b)~ Parameter $r$ is determined by the geometry of $\Delta$. }
  \label{..}
\end{figure}

At a high level, we consider two balls around each projected sample point $\hvec x \in \hat S_\parallel$ with appropriate choice of radii $r_1< r_2$ (see Figure~\ref{fig:noisysimplex}).
Consider the set of projections $\hat S_\parallel$ when points in $B_{r_2}(\vec x)$ are removed from it.
For points that are far from all $\vec a_i$, this set still includes points that are close to  $\vec a_i$  for all topics $i\in [k]$.
So, the convex hull of $\hat S_\parallel \setminus B_{r_2}(\vec x)$ is close to $\Delta$, and in particular, intersects $B_{r_1}(\vec x)$.
On the other hand, for $\vec x$ that is close to $\vec a_i$,  $\hat S_\parallel \setminus B_{r_2}(\vec x)$ does not include an extreme point of $\Delta$  or points close to it. So, 
the convex hull of $\hat S_\parallel \setminus B_{r_2}(\vec x)$ is considerably smaller than $\Delta$, and in particular, does not intersect $B_{r_1}(\vec x)$.

The geometry of the simplex and the angles between $\vec a_1, \dots, \vec a_k$ play an important role in choosing the appropriate $r_1$ and $r_2$. Note that when the samples are perturbed by noise, $\vec a_1, \dots, \vec a_k$
can only be approximately recovered if they are sufficiently far apart and the angles of the simplex at each  $\vec a_i$ is far from being flat. 
That is, we assume that for all $i\neq j$, $\| \vec a_i - \vec a_j \| \geq 3 \epsilon$. 
Furthermore, define $r\geq 1$ to be the smallest value such that
the distance between $\vec a_i$ and $\convh( \Delta \setminus B_{ r \epsilon}(\vec a_i))$ is at least $\epsilon$.
Note that such a value of $r$ always exists and depends entirely on the angles of the simplex defined by the class vectors. Therefore, the number of samples needed for our method depends on the value of $r$. The smaller the value of $r$, the larger is the separation between the topic vectors and the easier it is to identify them.
See Figure~\ref{fig:skew} for a demonstration of this concept.

\begin{claim} \label{claim:extreme-noisy}
Let $\epsilon' = \epsilon/8r$.
Let $\hat S_{\parallel}$ be the set of denoised projections, as in step~\ref{item:S||} of Algorithm~\ref{alg:noise-gaussian}.
For any $\hvec x\in \hat S_{\parallel}$ such that for all $i$, $\| \hvec x - \vec a_i \| > 8r\epsilon'$, 
$\dist(\hvec x,  \convh(\hat S_\parallel \setminus B_{6r\epsilon'}(\hvec x) ) ) \leq  2\epsilon'$.
Furthermore, for all $i\in[k]$ there exists $\hvec a_i\in \hat S_\parallel$  such that $\| \hvec a_i - \vec a_i \|  < \epsilon'$ and
$\dist(\hvec a_i , \convh(\hat S_\parallel \setminus B_{6r\epsilon'}(\hvec a_i ))) >   2\epsilon'$.
\end{claim}
\begin{proof}
Recall that by Lemma~\ref{lem:phase-denoise}, for any $\hvec x\in \hat S_\parallel$ there exists  $\vec x\in \Delta$ such that $\| \hvec x - \vec x\| \leq \epsilon'$ and for all $i\in[k]$, there exists $\hvec a_i \in \hat S_\parallel$ such that $\| \hvec a_i - \vec a_i\|\leq \epsilon'$.
For the first part, let $\vec x = \sum_i \alpha_i \vec a_i \in \Delta$ be the corresponding point to $\hvec x$, where $\alpha_i$'s are the coefficients of the convex combination.
Furthermore, let $\vec x' = \sum_i \alpha_i \hvec a_i$. 
We have,
\[ \| \vec x'  - \hvec x\| \leq \left\| \sum_{i=1}^k \alpha_i \hvec a_i  -  \sum_{i=1}^k \alpha_i \vec a_i   + \vec x - \hvec x \right\| \leq  \left\| \max_{i\in[k]} ~( \hvec  a_i  - \vec a_i) \right\|  + \left\|\vec x - \hvec x \right\| \leq 2\epsilon'.
\]
The first claim follows from the fact that $\| \hvec x - \vec a_i\| > 8r\epsilon'$ and as a result  $\vec x' \in  \convh(\hat S_\parallel \setminus B_{6r\epsilon'}(\hvec x))$.
Next, note that $B_{4r\epsilon'}(\vec a_i) \subseteq B_{5r\epsilon'} (\hvec a_i)$.
So, by the fact that $\| \vec a_i - \hvec a_i\| \leq \epsilon'$, 
\[ \dist \left(    \hvec a_i, \convh(\Delta \setminus B_{5r\epsilon'}(\hvec a_i) )  \right)\geq 
    \dist\left(   \vec a_i, \convh(\Delta \setminus B_{4r\epsilon'}(\vec a_i))    \right)  - \epsilon'\geq 3\epsilon'.
\]
Furthermore, we argue that if there is $\hvec x \in \convh(\hat S_\parallel\setminus B_{5r\epsilon'}(\hvec a_i))$ then there exists $\vec x\in \convh(\Delta \setminus B_{4r\epsilon'}(\hvec a_i))$, such that $\| \vec x - \hvec x\|\leq \epsilon'$. The proof of this claim is relegated to Appendix~\ref{app:CH_claim}. Using this claim, we have
$\dist\left(   \hvec a_i, \convh(\hat S_\parallel \setminus B_{6r\epsilon'}(\hvec a_i)) \right) \geq 2\epsilon'.
$
\end{proof}

Given the above structure, it is clear that set of points in $C$ are all within $\epsilon$ of one of the $\vec a_i$'s. So, we can cluster $C$ using single linkage with threshold $\epsilon$ to recover $\vec a_i$ up to accuracy $\epsilon$.

\section{Additional Results, Extensions, and Open Problems}
\subsection{Sample Complexity Lower bound} \label{sec:lower-bound}
As we observed the number of samples required by our method is $poly(n)$.
However, as the number of classes can be much smaller than the number of features, one might hope to recover $\vec v_1, \dots, \vec v_k$, with a number of samples that is polynomial in $k$ rather than $n$.
Here, we show that in the general case $\Omega(n)$ samples are needed to learn $\vec v_1, \dots, \vec v_k$ regardless of the value of $k$.

For ease of exposition, let $k=1$ and note that in this case every sample should be purely of one type.
Assume that the class vector, $\vec v$, is promised to be in the set $C = \{\vec v^j \mid v^j_\ell = 1/ \sqrt{2}, \text{ if } \ell = 2j-1 \text{ or } 2j, \text{ else } v^j_\ell = 0\}$.
Consider instances $(\vec  x_j^1, \vec x^2_j)$ such that the $\ell^{th}$ coordinate of $\vec x_j^1$ is  $x^1_{j \ell} = -1/\sqrt{2}$ if $\ell = 2j-1$ and $1/\sqrt{2}$ otherwise, and $x^2_{j\ell} = -1/\sqrt{2}$ if $\ell = 2j$ and $1/\sqrt{2}$ otherwise.
For a given $(\vec x_j^1, \vec x_j^2)$, we have that $\vec v^j \cdot \vec x_j^1 = \vec v^j \cdot \vec x_j^2 = 0$. On the other hand, for all $\ell\neq j$, $\vec v^\ell \cdot \vec x_j^1 = \vec v^\ell \cdot \vec x_j^2 = 1$.  Therefore, sample $(\vec x_j^1, \vec x_j^2)$ is consistent with $\vec v = \vec v^\ell$ for any $\ell\neq j$, but not with $\vec v = \vec v^j$. That is,  each instance $(\vec  x_j^1, \vec x^2_j)$ renders only one candidate of $C$ invalid.
Even after observing at most $\frac n 2 -2$ samples of this types, at least $2$ possible choices for $\vec v$ remain. So, $\Omega(n)$ samples are indeed needed to find the appropriate $\vec v$.
The next theorem, whose proof appears in Appendix~\ref{app:lower-bound} generalizes this construction and result to the case of any $k$.
\begin{theorem} \label{thm:lower}
For any $k \leq n$, any  algorithm that for all $i\in [k]$ learns $\vec v'_i$ such that $\| \vec v_i  - \vec v'_i \|_2\leq 1/\sqrt 2$, requires $\Omega(n)$ samples.
\end{theorem}
 
Note that in the above construction samples have large components in the irrelevant features. 
It would be interesting to see if this lower bound can be circumvented using additional natural assumptions in this model, such as assuming that the samples have length $\poly(k)$.

\subsection{Alternative Noise Models} \label{sec:agnostic}
Consider the problem of recovering $\vec v_1, \dots, \vec v_k$ in the presence of agnostic noise, where for an $\epsilon$ fraction of the samples $(\vec x^1, \vec x^2)$, $\vec x^1$ and $\vec x^2$ correspond to different mixture weights.
% $\vec w_1$ and $\vec w_2$.
Furthermore, assume that the distribution over the instance space is rich enough such that any subspace other than  $\spn\{\vec v_1, \dots, \vec v_k\}$ is inconsistent with a set of instances of non-negligible density.\footnote{This assumption is similar to the richness assumption made in the standard case, where we assume that there is enough ``entropy'' between the two views of the samples such that even in the non-noisy case the subspace can be uniquely determined by taking the nullspace of $X_1 - X_2$.}
Since the VC dimension of the set of $k$ dimensional subspaces in $\R^n$ is $\min\{ k , n-k\}$, from the information theoretic point of view, one can recover $\spn\{ \vec v_1, \dots, \vec v_k\}$ as it is the only subspace that is inconsistent with less than $O(\epsilon)$ fraction of  $\tilde O(\frac{k}{\epsilon^2})$ samples. Furthermore, we can detect and remove any noisy sample, for which the two views of the sample are not consistent with $\spn\{ \vec v_1, \dots, \vec v_k \}$. And finally, we can recover $\vec a_1, \dots, \vec a_k$ using phase 2 of Algorithm~\ref{alg:noisefree}.

In the above discussion, it is clear that once we have recovered $\spn\{ \vec v_1, \dots, \vec v_k\}$,  denoising and finding the extreme points of the projections can be done in polynomial time.
For the problem of recovering a $k$-dimensional nullspace, \cite{hardt2013algorithms} introduced an efficient algorithm that tolerates agnostic noise up to $\epsilon = O(k/n)$. Furthermore, they provide an evidence that this result might be  tight. It  would be interesting to see whether additional structure present in our model, such as the fact that samples are  convex combination of classes, can allow us to efficiently recover the nullspace in presence of more noise.

Another interesting open problem is whether it is possible to handle the case of $p_0=0$.
%Another interesting noise model is an extension of our results for $p_0 =0$. 
That is, when \emph{every document} is affected by Gaussian noise $\N(0, \sigma^2I_n)$, for $\sigma \gg \epsilon$. A simpler form of this problem is as follows. Consider a distribution induced by first drawing $\vec x \sim D$, where $D$  is an arbitrary and unknown distribution over $\Delta = \convh(\{\vec a_1, \dots, \vec a_k\})$, and taking $\hvec x = \vec x + \N(0, \sigma^2I_n)$. \emph{Can we learn $\vec a_i$'s within error of $\epsilon$ using polynomially many samples?} 
Note that when $D$ is only supported on the corners of $\Delta$, this problem reduces to learning mixture of Gaussians, for which there is a wealth of literature on estimating Gaussian means and mixture weights~\citep{dasgupta2002pac,kalai2012disentangling,moitra2010settling}. It would be interesting to see under what regimes $\vec a_i$ (and not necessarily the mixture weights) can be learned when $D$ is an arbitrary distribution over $\Delta$.

\subsection{General function $f(\cdot)$}

Consider the general model described in Section~\ref{sec:model}, where $f_i(x) = f(\vec v_i \cdot \vec x)$ for an unknown strictly increasing function $f:\R^+ \rightarrow [0,1]$ such that $f(0) = 0$. 
We describe how variations of the techniques discussed up to now can extend to this more general setting.

For ease of exposition, consider the non-noisy case. 
Since $f$ is a strictly increasing function, $f(\vec v_i \cdot \vec x^1) = f(\vec v_i \cdot \vec x^2)$ if and only if
$\vec v_i \cdot \vec x^1 = \vec v_i \cdot \vec x^2$. Therefore, we can recover $\spn(\vec v_1, \dots, \vec v_k)$ by the same approach as in Phase 1 of Algorithm~\ref{alg:noisefree}.
Although, by definition of pseudoinverse matrices, the projection of $\vec x$ is still represented by $\vec x_\parallel = \sum_i (\vec v_i \cdot \vec x) \vec a_i$, this is not necessarily a convex combination of  $\vec a_i$'s anymore. This is due  to the fact that $\vec v_i \cdot \vec x$ can add up to values larger than $1$ depending on $\vec x$. 
However, $\vec x_\parallel$ is still a \emph{non-negative combination} of $\vec a_i$'s.
Moreover, $\vec a_i$'s are linearly independent, so $\vec a_i$ can not be expressed by a nontrivial non-negative combination of other samples. Therefore, for all $i$, $\vec a_i / \| \vec a_i \|$ can be recovered by taking \emph{the extreme rays of the convex cone} of the projected samples. So, we can recover $\vec v_1, \dots, \vec v_k$, by taking the psuedoinverse of $\vec a_i / \| \vec a_i \|$ and re-normalizing the outcome such that $\| \vec v_i\|_2=1$. When samples are perturbed by noise, a similar argument that also takes into account the smoothness of $f$ proves similar results.

It would be interesting to see whether a more general class of similarity functions, such as kernels, can be also learned in this context.

\bibliographystyle{apalike}
\bibliography{cotrain}

\appendix
\section{Omitted Proof from Section~\ref{sec:no-noise} --- No Noise} \label{app:no-noise}
\subsection{Proof of Lemma~\ref{lem:rank}}  \label{app:no-noise-rank}
For all $j\leq n-k$, let $Z_{j} = \{(\vec x^1_i - \vec x^2_i)\mid  i \leq \frac{j}{\zeta} \ln\frac n \delta\}$.
We prove by induction that for all $j$, $\rank(Z_j) < j$ with probability at most $j\frac{\delta}{n}$.

For $j=0$, the claim trivially holds.
Now assume that the induction hypothesis holds for some $j$.
Furthermore, assume that $\rank(Z_j) \geq j$. Then, $\rank(Z_{j+1}) < j+1$ only if the additional $\frac{1}{\zeta} \ln\frac n \delta$ samples in $Z_{j+1}$ all belong to $\spn(Z_j)$.
Since, the space of such samples has rank $< n-k$, this happens with probability at most $(1 - \zeta)^{\frac{1}{\zeta} \ln\frac n \delta} \leq \frac \delta n$. Together with the induction hypothesis that $\rank(Z_j) \geq j$ with probability at most $j\frac{\delta}{n}$, we have that 
$\rank(Z_{j+1}) < j+1$ with probability at most $\frac{(j+1) \delta}{n}$.
Therefore $\rank(Z) = \rank(Z_{n-k}) = n-k$ with probability at least $1-\delta$.

\subsection{Proof of Lemma~\ref{lem:sum-alpha-u}} \label{app:no-noise-sum-alpha-u}
First note that $V$ is a the pseudo-inverse of $A$, so their span is equal. Hence,  $\sum_{i\in[k]}(\vec v_i \cdot \vec x)  \vec a_i \in \spn\{ \vec v_1, \dots, \vec v_k\} $. It remains to show that $\left( \vec x - \sum_{i\in[k]}(\vec v_i \cdot \vec x)  \vec a_i  \right) \in \nul\{\vec v_1, \dots, \vec v_k\}$. We do so by showing that this vector is orthogonal to $\vec v_j$ for all $j$. We have
\begin{align*}
\left( \vec x - \sum_{i=1}^k (\vec v_i \cdot \vec x) \vec a_i \right) \cdot \vec v_j &= \vec x \cdot \vec v_j-  \sum_{i=1}^k (\vec v_i \cdot \vec x) (\vec a_i \cdot \vec v_j) \\
 &  = \vec x \cdot \vec v_j - \sum_{i\neq j}   (\vec v_i \cdot \vec x)    (\vec a_i\cdot \vec v_j) -   (\vec v_j \cdot \vec x)    (\vec a_j\cdot \vec v_j) \\
 &= \vec x \cdot \vec v_j  - \vec x \cdot \vec v_j  = 0.
\end{align*}
Where, the second equality follows from the fact when $A = V^+$, for all $i$, $\vec a_i\cdot \vec v_i = 1$ and $\vec a_j\cdot \vec v_i = $ for $j\neq i$.
Therefore, $\sum_{i\in[k]}(\vec v_i \cdot \vec x)  \vec a_i$ is the projection of $\vec x$ on $\spn\{\vec v_1, \dots, \vec v_k\}$.

\subsection{Proof of Lemma~\ref{lem:extreme-no-noise}}
\label{app:no-noise-extreme-no-noise}
Assume that $S$ included samples that are purely of type $i$, for all $i\in[k]$. That is, for all $i\in [k]$ there is $j\leq m$, such that $\vec v_i \cdot \vec x_j^1 = \vec v_i \cdot \vec x_j^2=1$ and $\vec v_{i'} \cdot \vec x_j^1 = \vec v_{i'} \cdot \vec x_j^2=0$ for $i' \neq i$.
By Lemma~\ref{lem:sum-alpha-u}, the set of projected vectors form the set $ \{ \sum_{i=1}^k (\vec v_i \cdot \vec x_j) \vec a_i  \mid j\in [m] \}$.
Note that $\sum_{i=1}^k (\vec v_i \cdot \vec x_j) \vec a_i$ is in the simplex with vertices $\vec a_1, \dots, \vec a_k$. Moreover, for each $i$,  there exists a pure sample in $S$ of type $i$. Therefore, $\convh\{ \sum_{i=1}^k  (\vec v_i \cdot \vec x_j) \vec a_i \mid j\in [m] \}$ is the simplex on linearly independent vertices $\vec a_1, \dots, \vec a_k$. As a result, $\vec a_1, \dots, \vec a_k$ are the extreme points of it.

It remains to prove that with probability $1-\delta$,  the sample set has a document of purely type $j$, for all $j\in[k]$. By the assumption on the probability distribution $\P$, with probability at most $(1- \xi)^m$,  there is no document of type purely $j$. Using the union bound, we get the final result.

\section{Technical Spectral Lemmas} \label{app:spectral}

\begin{prop}[\cite{davis1970rotation} $\sin\theta$ theorem]. \label{prop:kahan}
Let $B, \hat B \in \R^{p\times p}$ be symmetric, with  eigen values $\lambda_1 \geq \cdots \geq \lambda_p$ and
$\hat \lambda_1 \geq \cdots \geq \hat \lambda_p$, respectively.
Fix $1 \leq r \leq s \leq p$ and let $V = (\vec v_r, \dots, \vec v_s)$ and $\hat V = ({\hat{\vec v}}_r , \dots, {\hat{\vec v}}_s)$ be the orthonormal eigenvectors corresponding to $\lambda_r, \dots, \lambda_s$ and $\hat \lambda_r, \dots,\hat \lambda_s$.
Let $\delta = \inf \{ |\hat \lambda - \lambda |: \lambda \in [\lambda_s, \lambda_r], \hat \lambda \in (-\infty, \hat \lambda_{s-1}] \cup [\hat \lambda_{r+1}, \infty) \} > 0$. Then ,
\[ \| \sin \Theta(V, \hat V) \|_2 \leq \frac{ \| \hat B - B \|_2 }{\delta}.
\]
where $\sin \Theta(V, \hat V) = P_V - P_{\hat V}$, where $P_V$ and $P_{\hat V}$ are the projection matrices for  $V$ and $\hat V$.
\end{prop}

\begin{prop}[Corollary 5.50~\citep{vershynin2010introduction}] \label{prop:covariance-gauss}
Consider a Gaussian distribution in $\R^n$ with co-variance matrix $\Sigma$. Let $A \in \R^{n\times m}$ be a matrix whose rows are drawn i.i.d from this distribution, and let $\Sigma_m = \frac 1m A A^\top$. For every $\epsilon \in (0,1)$, and $t$, if $m \geq c n (t / \epsilon)^2 $ for some constant $c$, then with probability at least $1 - 2 \exp(-t^2 n)$, $\| \Sigma_m - \Sigma \|_2 \leq \epsilon \| \Sigma \|_2$
\end{prop}

\begin{prop}[Matrix Bernstein~\citep{tropp2015introduction}] \label{prop:bernstein}
Let $S_1, \dots, S_n$ be independent, centered random matrices with
common dimension $d_1 \times d_2$, and assume that each one is uniformly bounded. That is, $\E S_i = 0$ and $\| S_i \|_2 \leq L$ for all $i\in[n]$.
Let $Z = \sum_{i=1}^n S_i$, and let $v(Z)$ denote the matrix variance:
\[   v(Z) = \max \left\{ \left\| \sum_{i=1}^n \E[S_i S_i^\top] \right\|,  \left\| \sum_{i=1}^n \E[S_i^\top S_i] \right\|   \right\}.
\]
Then, 
\[
\P [ \|Z \| \geq t ] \leq (d_1+d_2) \exp \left(\frac{-t^2/2}{ v(Z) + Lt/3} \right).
 \]
\end{prop}

\begin{prop}[Theorem 4.10 of \cite{stewart1990matrix}] \label{prop:diffeigen}
Let $\hat A = A + E$ and let $\lambda_1, \dots, \lambda_n$ and $\lambda'_1, \dots, \lambda'_n$ be the eigen values of $A$ and $A+E$. Then, $\max\{ | \lambda'_i - \lambda_i |\} \leq \| E\|_2$.
\end{prop}

\begin{prop}[Theorem 3.3 of \cite{stewart77perturbation}]\label{prop:inverseperturb}
For any $A$ and $B = A + E$, $$\| B^+ - A^+\| \leq \max 3 \left\{ \| A^{+} \|^2, \|B^{+}\|^2  \right\} \|E\|,$$
where  $\|\cdot\|$ is an arbitrary norm.
\end{prop}

\section{Omitted Proof from Section~\ref{sec:phase1} --- Phase 1}
\subsection{Proof of Claim~\ref{claim:DE-estimate}} \label{app:claim:DE-estaimte}
Let $\vec e_i$ and  $\vec d_i$ be the $i^{th}$ row of $E$ and $D$. Then $E D^\top = \sum_{i=1}^m \vec e_i  \vec d_i^\top$ and $D E^\top = \sum_{i=1}^m \vec d_i  \vec e_i^\top$.
Let $S_i =  \frac 1m \begin{bmatrix} 0                       &     \vec e_i  \vec d_i^\top\\
                                          \vec d_i  \vec e_i^\top    &   0   \\
                   \end{bmatrix}$.
Then,   $\|  \frac 1m D E^\top + \frac 1m E D^\top   \|_2 \leq 2 \| \sum_{i=1}^m  S_i \|_2$. We will use matrix Bernstein to show that $\sum_{i\in[m]} S_i$ is small with high probability.

First note that the distribution of $\vec e_i$ is a Gaussian centered at $0$, therefore, $\E[S_i] = 0$. 
Furthermore, for each $i$, with probability $1-\delta$, $\|\vec  e_i\|_2 \leq  \sigma \sqrt{n} \log\frac 1\delta$. So, with probability $1-\delta$, for all samples $i \in [m]$,  $\| \vec e_i\|_2 \leq  \sigma \sqrt{n} \log\frac m\delta$. Moreover, by assumption $\|\vec d_i \|=\| \vec x_i^1 - \vec x_i^2\| \leq 2M$.
Therefore, with probability $1-\delta$, 
\[
L = \max_i \|S_i\|_2 = \frac 1m \max_i \|\vec e_i \| \|\vec d_i\| \leq  \frac{2}{m} \sigma \sqrt{n} M ~\polylog\frac{n}{\epsilon\delta}.
\]

Note that, $\left\|  \E[S_i S_i^\top] \right\| = \frac{1}{m^2} \left\|  \E[(\vec e_i\vec d_i^\top )^2]  \right\| \leq L^2.$
Since $S_i$ is Hermitian, the matrix covariance defined by Matrix Bernstein inequality is
\[
v(Z) = \max \left\{ \left\| \sum_{i=1}^m \E[S_i S_i^\top] \right\|,  \left\| \sum_{i=1}^m \E[S_i^\top S_i] \right\|   \right\} = \left\| \sum_{i=1}^m \E[S_i S_i^\top] \right\| \leq m L^2.
\]

If  $\epsilon \leq v(Z) / L$ and $m\in \Omega( \frac{n \sigma^2 M^2}{\epsilon^2} \polylog\frac{n}{\epsilon\delta} )$
or  $\epsilon \geq v(Z) / L$ and $m\in \Omega( \frac{\sqrt n \sigma M}{\epsilon} \polylog\frac{n}{\epsilon\delta} )$, 
 using Matrix Bernstein inequality (Proposition~\ref{prop:bernstein}), we have
\[ \Pr\left[ \left\| \frac 1m D E^\top + \frac 1m E D^\top \right\| \geq \epsilon \right] =  \Pr\left[ \left\| \sum_{i=1}^m  S_i \right\| \geq \frac \epsilon 2  \right] \leq \delta.\]

\subsection{Proof of Claim~\ref{claim:DD-estimate}} \label{app:estimateDD}
Let $\vec d_i$ be the $i^{th}$ row $D$. Then $D D^\top = \sum_{i=1}^m \vec d_i  \vec d_i^\top$.
Let $S_i =  \frac 1m \vec d_i \vec d_i^\top - \frac 1m \E[\vec d_i \vec d_i^\top]$.
Then, $\|  \frac 1m D D^\top  - \E\left[ \frac 1m D D^\top  \right] \|_2 =  \| \sum_{i=1}^m  S_i \|_2$.
Since, $\vec d_i = \vec x_i^1 - \vec x_i^2$ and $\|\vec x_i^j\|\leq M$, we have that for any  $i$, $\| \vec d_i \vec d_i^\top -  \E[\vec d_i  \vec d_i^\top] \| \leq 4M^2$. Then, 
\[   L = \max_i \|S_i\|_2 = \frac 1m \max_i \| \vec d_i \vec d_i^\top - \E[\vec d_i \vec d_i^\top] \|_2 \leq \frac 4m M^2,
\]
and $\| \E[S_i S_i^\top] \leq L^2$. Note that $S_i$ is Hermitian, so, the matrix covariance is 
\[
v(Z) = \max \left\{ \left\| \sum_{i=1}^m \E[S_i S_i^\top] \right\|,  \left\| \sum_{i=1}^m \E[S_i^\top S_i] \right\|   \right\} = \left\| \sum_{i=1}^m \E[S_i S_i^\top] \right\| \leq m L^2.
\]

If $\delta_0 \leq 4M^2$ and $m \in \Omega( \frac{M^4}{\delta_0^2} \log \frac n \delta)$ or $\delta_0 \geq 4M^2$ and $m \in \Omega( \frac{M^2}{\delta_0} \log \frac n \delta)$, then by   Matrix Bernstein inequality (Proposition~\ref{prop:bernstein}), we have
\[ \Pr\left[ \left\| \sum_{i=1}^m  S_i \right\| \geq \frac {\delta_0}{ 2} \right] \leq \delta.\]

\section{Omitted Proof from Section~\ref{sec:denoise} --- Denoising}

\subsection{Proof of Claim~\ref{claim:high-dense}} \label{app:claim:high-density}

Recall that for any $i\in[k]$, with probability $\pg= g(\near)$ a nearly pure weight vector $\vec w$ is generated from $\P$, such that $\| \vec w - \vec e_i\| \leq \near$.
And independently, with probability $p_0$ the point is not noisy. Therefore, there is $p_0 \pg$ density on non-noisy points that are almost purely of class $i$. Note that for such points, $\vec x$, 
\[
\| P\vec x - \vec a_i \|  = \left\| \sum_{j=1}^k w_j \vec a_j - \vec a_i \right\| \leq k ( \near)(\sizeA) \leq  \frac{\epsilon'}{8}.
\]
Since $\| P - \hat P \| \leq \epsilon' / 8M$, we have
\[ \| \vec a_i - \hat P \vec x \| = \| \vec a_i -  P \vec x \|  + \|  P \vec x - \hat P \vec x \|  \leq \frac{\epsilon'}{8} +\frac{\epsilon'}{8} \leq \frac{\epsilon'}{4}
\]
The claim follows immediately. 

\section{Omitted Proof from Section~\ref{sec:phase2} --- Phase 2}

\subsection{Omitted proof from Claim~\ref{claim:extreme-noisy}}
\label{app:CH_claim}
Here, we prove that  $\hvec x \in \convh(\hat S_\parallel\setminus B_{d+\epsilon'}(\hvec a_i))$ then there exists $\vec x\in \convh(\Delta \setminus B_{d}(\hvec a_i))$, such that $\| \vec x - \hvec x\|\leq \epsilon'$.

Let $\vec x = \sum_i \alpha_i \hvec z_i$ be the convex combination of $\hvec z_1, \dots, \hvec z_\ell  \in \hat S_\parallel  \setminus B_{d+\epsilon'}(\hvec a_i)$.
By Claim~\ref{lem:phase-denoise}, there are $\vec z_1, \dots, \vec z_\ell\in \Delta$, such that $\|\vec z_i - \hvec z_i\|\leq \epsilon'$ for all $i\in[k]$. Furthermore, by the proximity of $\vec z_i$ to $\hvec z_i$ we have that $\vec z_i \not\in B_{d}(\hvec a_i)$. Therefore, 
$\vec z_1, \dots, \vec z_\ell\in \Delta \setminus B_{d}(\hvec a_i)$. Then,
$\vec x = \sum_i \alpha_i \vec z_i$ is also within distance $\epsilon'$.

\section{Proof of Theorem~\ref{thm:lower} --- Lower Bound} \label{app:lower-bound}
For ease of exposition assume that $n$ is a multiple of $k$. Furthermore, in this proof we adopt the notion $(\vec x_i, \vec x'_i)$ to represent the two views of the $i^{th}$ sample. For any vector $\vec u\in \R^n$ and $i\in [k]$, we use $(\vec u)_i$ to denote the $i^{th}$ $\frac nk$-dimensional block of $\vec u$, i.e., coordinates $u_{(i-1)\frac nk+1},\dots, u_{i\frac nk}$.

Consider the $\frac n k$-dimensional vector $\vec u_j$, such that $u_{j\ell} = 1$ if $\ell = 2j-1$ or $2j$, and $u_{j\ell} = 0$, otherwise.
And consider $\frac n k$-dimensional vectors $\vec z_j$ and  $\vec z'_j$, such that $z_{ j\ell } = -1$ if $\ell = 2j-1$ and $z_{ j\ell } =1$ otherwise, and $z'_{j \ell} = -1$ if $\ell = 2j$ and $z'_{j \ell} = 1$ otherwise.
Consider a setting where $\vec v_i$ is restricted to the set of  candidate $C_i = \{\vec v^j_i \mid (\vec v^j_i)_i = \vec u_j /\sqrt{2} \text{ and }  (\vec v^j_i)_{i'} = \vec 0 \text{ for } i'\neq i \}$. In other words, the $\ell^{th}$ coordinate of $\vec v^j_i$ is $1/\sqrt{2}$ if $\ell = (i-1)\frac nk+2j-1$ or $(i-1)\frac nk+2j$, else $0$.
Furthermore, consider instances $( \vec x^j_i, \vec x'^j_i)$ such that $(\vec x^j_i)_i = \vec z_j / \sqrt{2}$ and  $(\vec x'^j_i)_i = \vec z'_j/\sqrt{2}$ and for all $i'\neq i$,  $(\vec x^j_i)_{i'} =  (\vec x'^j_i)_{i'} = \vec 0$. In other words,
\begin{align*}
\vec x^j_i &= \frac {1}{\sqrt 2}~(0, \dots, 0,    \ \   1, \dots, 1, \!\!\!\!\!\!\!\!\!\!\!\!\!\!\!\! \overbrace{1 ,  -1 }^{(i-1)\frac nk+2j -1, (i-1)\frac nk + 2j}  \!\!\!\!\!\!\!\!\!\!\!\!\!\!\!\!\!\!, 1, \dots, 1,\ \  0, \dots, 0),\\
\vec x'^j_i  &= \frac {1}{\sqrt 2}~(0, \dots, 0, \ \ 1, \dots, 1,  -1 ,\ \  1\  ,  1, \dots, 1, \ \ 0, \dots, 0),\\
\vec v^j_i  &= \frac {1}{\sqrt 2}~(0, \dots, 0, \   \underbrace{\ 0, \dots, 0, \ \ \,  1 ,\ \  1\  , 0 , \dots,0, \ }_{i^{th}\ block} \  0, \dots, 0).
\end{align*}

First note that, for any $i, i'\in [k]$ and any $j, j'\in [\frac{n}{2k}]$, $\vec v_i^j \cdot \vec x^{j'}_{i'} = \vec v_i^j \cdot \vec x'^{j'}_{i'}$. That is, the two views of all instances are consistent with each other with respect to all candidate vectors.
Furthermore, for any $i$ and $i'$ such that $i\neq i'$, for all $j, j'$, $\vec v_i^j \cdot \vec x^{j'}_{i'} = 0$. Therefore, for any observed sample $(\vec x^j_i, \vec x'^j_i)$, the sample should be purely of type $i$.

For a given $i$, consider all the samples $(\vec x^j_i, \vec x'^j_i)$ that are observed by the algorithm.
Note that $\vec v_i^j \cdot \vec x^j_i = \vec v_i^j \cdot \vec x'^j_i =0$. And for all $j' \neq j$,  $\vec v_i^{j'} \cdot \vec x^j_i = \vec v_i^{j'} \cdot \vec x'^j_i =1$. Therefore,  observing  $(\vec x^j_i, \vec x'^j_i)$ only rules out $\vec v_i^j$ as a candidate, while this sample is consistent  with candidates $\vec v_i^{j'}$ for $j'\neq j$. Therefore, even after observing $\leq \frac {n}{2k} -2$ samples of this types, at least $2$ possible choices for $\vec v_i$ remain valid. Moreover, the distance between any two $\vec v^j_i,\vec v^{j'}_i \in C_i$ is $\sqrt{2}$. Therefore, $\frac {n}{2k} -1$ samples are needed to learn $\vec v_i$ to an accuracy better than $\sqrt{2}/2$.

Note that consistency of the data with $\vec v_{i'}$ is not affected by the samples of type $\vec x_i^j$ that are observed by the algorithms when $i'\neq i$. So, $\Omega(k \frac nk)= \Omega(n)$ samples are required to approximate all $\vec v_i$'s to an accuracy better than $\sqrt{2}/2$.

%\subsection{Proof of the Claim from Lemma~\ref{lem:extreme-noisy}}
%\label{app:CH_claim}
%Here, we prove that  $\hat x \in \convh(\hat S_\parallel\setminus B_{d+\epsilon'}(\hvec a_i))$ then there exists $\vec x\in \convh(\Delta \setminus B_{d}(\vec a'_i))$, such that $\| \vec x - \hvec x\|\leq \epsilon'$.
%
%Let $\vec x = \sum_i \alpha_i \hvec z_i$ be the convex combination of $\hvec z_1, \dots, \hvec z_\ell  \in \hat S_\parallel  \setminus B_{d+\epsilon'}(\hvec a_i)$.
%By Lemma~\ref{lem:denoised}, there are $\vec z_1, \dots, \vec z_\ell\in \Delta$, such that $\|\vec z_i - \hvec z_i\|\leq \epsilon'$ for all $i\in[k]$. Furthermore, by the proximity of $\vec z_i$ to $\hvec z_i$ we have that $\vec z_i \not\in B_{d}(\hvec a_i)$. Therefore, 
%$\vec z_1, \dots, \vec z_\ell\in \Delta \setminus B_{d}(\vec a'_i)$. Then,
%$\vec x = \sum_i \alpha_i \vec z_i$ is also within distance $\hvec x$.

\end{document}